\documentclass{article}

\usepackage{amsmath}

\usepackage{xcolor}

\usepackage{tikz}
\usepackage[utf8]{inputenc} % allow utf-8 input
\usepackage[T1]{fontenc}    % use 8-bit T1 fonts
\usepackage{hyperref}       % hyperlinks
\usepackage{url}            % simple URL typesetting
\usepackage{booktabs}       % professional-quality tables
\usepackage{amsfonts}       % blackboard math symbols
\usepackage{nicefrac}       % compact symbols for 1/2, etc.
\usepackage{microtype}      % microtypography
\usepackage[capitalize,noabbrev]{cleveref}      % smart cross-referencing
\usepackage{lipsum}         % Can be removed after putting your text content
\usepackage{graphicx}
\usepackage{natbib}
\usepackage{doi}
\usepackage{amssymb}
\usepackage{mathtools}
\usepackage{amsthm}
\usepackage{enumitem}
\usepackage{gensymb}

\newtheorem{mydef}{Definition}
\newtheorem{mythm}{Theorem}

\newtheorem{mylem}{Lemma}
\newtheorem{mycor}{Corollary}
\usepackage{authblk}
\newcommand{\ignore}[1]{}

\renewcommand{\hom}{\mathrm{Hom}}
\newcommand{\cat}{\mathcal{C}}
\newcommand{\catf}{\mathcal{C}_f}
\newcommand{\cx}{\mathcal{X}}
\newcommand{\cy}{\mathcal{Y}}
\newcommand{\ob}{\mathrm{Ob}}
\newcommand{\Hom}{\mathrm{Hom}}
\newcommand{\id}{\mathrm{id}}
\newcommand{\op}{\mathrm{op}}
\newcommand{\E}{\mathbb{E}}
\newcommand{\set}{\textbf{Set}}
\newcommand{\cop}{{\cat^\mathrm{op}}}

\usepackage{arxiv}

\title{On the Power of Foundation Models}

\date{}

\author[1,2,3]{Yang Yuan}

\affil[1]{\footnotesize IIIS, Tsinghua University}
\affil[2]{\footnotesize Shanghai Artificial Intelligence Laboratory}
\affil[3]{\footnotesize Shanghai Qi Zhi Institute}

\hypersetup{
pdftitle={On the Power of Foundation Models},
pdfauthor={Yang Yuan},
pdfkeywords={Foundation Model, Category Theory, Yoneda Lemma,
Downstream Task
},
}
\renewcommand{\headeright}{}
\renewcommand{\undertitle}{}
\renewcommand{\shorttitle}{On the Power of Foundation Models}
\newcommand{\questionwidth}{0.8\textwidth}
\newcommand{\arxivbox}{\mbox}
\newcommand{\enumerateLeftMargin}{0.5cm}
\begin{document}
\maketitle

\begin{abstract}
With infinitely many high-quality data points, infinite computational power, an infinitely large foundation model with a perfect training algorithm and guaranteed zero generalization error on the pretext task, can the model be used for everything?
This question cannot be answered by the existing theory of representation, optimization or generalization, because the issues they mainly investigate are assumed to be nonexistent here. 
In this paper, we show that category theory provides powerful machinery to answer this question. We have proved three results. The first one limits the power of prompt-based learning, saying that the model can solve a downstream task with prompts if and only if the task is representable. 
The second one says fine tuning does not have this limit, as a foundation model with the minimum required power (up to symmetry) can theoretically solve downstream tasks  for the category defined by pretext task, with fine tuning and enough resources. 
Our final result can be seen as a new type of generalization theorem, showing that the foundation model can represent unseen objects from the target category (e.g., images) using the structural information from the source category (e.g., texts). Along the way, we provide a categorical 
framework for supervised and self-supervised learning, which might be of independent interest. 
\end{abstract}

\section{Introduction}
Foundation models have recently exhibited remarkable proficiency in addressing a myriad of complex downstream tasks that were very difficult or impossible for the previous methods~\citep{ramesh2021zero,rombach2022high,ramesh2022hierarchical,sohl2015deep,brown2020language, radford2018improving, radford2019language,he2022masked}. Being different in network structure, dataset and training algorithms, the foundation models are similar in terms of the training process: the model is first trained with a large unlabeled dataset on a pretext task, and then being applied to other downstream tasks with the parameters frozen. Freezing the parameter is necessary because training a foundation model is very expensive, and the downstream tasks usually have limited data samples, which makes retraining less attractive.

Given a set of  downstream tasks, which of them are solvable by the foundation models, and which are not? This is a very fundamental question. The existing theory attacks this question through various perspectives like dataset quality and quantity, computational power, network structure, etc. However, as we collect trillions of data points, build gigantic GPU data centers, and optimize huge networks with billions of parameters, a simple question just pops up:

\begin{center}
\textbf{Where are we heading to?}
\end{center}

Specifically, with infinitely many high-quality data points, infinite computational power, an infinitely large foundation model with a perfect training algorithm and guaranteed zero generalization error on the pretext task, can the model be used for every downstream task?
As we will see, this question is not about model representation, optimization or generalization, but structural representation of the tasks. Category theory, as the theory of mathematical structures, is the ideal tool for answering this question. 

For the downstream tasks, there are mainly two types of methods: prompt tuning and fine tuning. Prompt tuning does not train with the downstream tasks. Instead, it only sends a task specific prompt to the model, so that the model can ``switch'' its working mode for solving the task. Fine tuning trains a small network connecting to the foundation model with the labeled dataset for the downstream task. 

In Theorem~\ref{thm:prompt-tuning}, we show that with prompt tuning, the model can solve the task if and only if the task is ``representable'' in the category defined by the pretext task. If the category  does not have complicated structures, our theorem indicates that the power of prompt tuning is  limited. For example, in Corollary~\ref{cor:rotate}, we show that the rotation pretext task~\cite{gidaris2018unsupervised} alone does not guarantee that prompt tuning can
solve complicated downstream tasks such as segmentation or classification.

On the other hand, the results on fine tuning are more promising. Our Theorem~\ref{thm:fine-tuning} proves that for the foundation model with the minimum required power (up to symmetry) for the pretext task and enough resources including training data, it can potentially solve any downstream tasks for the category defined by pretext task. The role of pretext task is crucial in the sense that if the pretext task fails to extract adequate information from the unlabeled dataset, the power of fine tuning remains restricted.

Along the way, we have provided a categorical framework for machine learning. Interestingly, the framework injects the learning perspective to category theory as well. Therefore, we also proved a generalization theorem for structural learning (Theorem~\ref{thm:generalization}), which explains why self-supervised learning for text-image generation tasks can represent image objects such as avocado chair, which do not exist in the dataset or real world, and can support their generation when combined with a target-side decoder or generator.
Theorem~\ref{thm:generalization} can be easily generalized to the compositional theorem of multiple categories, see Theorem~\ref{thm:compose}.

Unlike most machine learning theory papers, our paper does not have any assumptions on the data distribution or network structure. Instead, we take the bird's-eye view that 
is model oblivious, and only focuses on the structure defined by the pretext task. It is indeed possible that by designing a special network, one may get a more powerful model with better performance. However, we stick with our setting because: 
\begin{itemize}[leftmargin=0.3cm, parsep=0cm, topsep=0cm]
	\item 
\textbf{Empirically, people do not customize network structures for different tasks.} Instead, they tend to use similar structures like ResNet~\cite{he2016deep} or Transformer~\cite{vaswani2017attention}. 
By the no free lunch theorem~\cite{shalev2014understanding}, 
if the model does not contain task-specific prior information, it will not 
be able to completely solve all the tasks. 
In other words, the standard models are not universally competent, and it is likely that the limitation that we derived from the model oblivious setting, also applies to the practical settings.  
\item \textbf{Pretext task design is a central problem in self-supervised learning.}
Currently, there are various datasets floating around, and a wide range of diverse tasks to solve, but the practitioners do not really know what kinds of pretext tasks to pick for solving a given task, or the limitations of each pretrained model. They get the intuition by trial and error with experiments, which are both expensive and noisy.  
Our framework will help them to think about this problem in a more mathematical and systematic way, and provide guidance for better pretext task design. 
\end{itemize}

\section{Related Work}
\label{sec:related}
\subsection{Self-supervised learning}
\textbf{Self-supervised learning.}
Recently, researchers have proposed many self-supervised learning algorithms for foundation models, including 
contrastive methods~\citep{
chen2020simple, he2020momentum, grill2020bootstrap, 
chen2021exploring, noroozi2016unsupervised,barlow}, 
masked image models~\citep{
he2022masked,dosovitskiy2020image, doersch2015unsupervised, pathak2016context}, 
masked language models~\citep{
devlin2018bert, raffel2020exploring}, 
pure language models~\citep{brown2020language,radford2018improving, radford2019language}, and with other pretext tasks~\citep{
oord2018representation, 
gidaris2018unsupervised, 
clark2020electra, noroozi2016unsupervised, 
pathak2017learning}.

\textbf{Multimodal learning.}
Self-supervised learning can also be applied to multimodal learning, including text + 
image~\citep{ramesh2021zero,rombach2022high,ramesh2022hierarchical,sohl2015deep}, 
video + audio~\citep{arandjelovic2018objects}. 
For generating images for the multimodal tasks,   
diffusion model is the state-of-the-art approach~\citep{rombach2022high, sohl2015deep, dhariwal2021diffusion, ho2020denoising}. 

\textbf{Prompt tuning.}
There are various prompt tuning methods, including the discrete prompts~\citep{brown2020language, jiang2020can, shin2020autoprompt, gao2020making} and continuous prompts~\citep{liu2021gpt, li2021prefix}.

\subsection{Theory for deep learning and self-supervised learning}
\textbf{Theory for deep learning} has been an active research area recently. Optimization theory focuses on how and why first order methods like stochastic gradient descent finds the local/global optimum of the neural networks~\citep{du2019gradient, arora2019fine, du2018gradient, allen2019convergence, allen2019learning, zou2020gradient, li2017convergence}. Generalization theory focuses on how the performance of the model in the training set transfers to the population distribution~\arxivbox{\citep{allen2019learning,arora2019fine, bartlett2017spectrally, bartlett2020benign, yin2019rademacher}}. Representation theory focuses on the representation power of the neural networks~\citep{hornik1989multilayer, cybenko1989approximation, raghu2017expressive}. 
There are also theoretical results on analyzing various aspects of reinforcement learning~\citep{du2019provably, du2019good, jin2018q, cai2020provably}.

\textbf{Theory for self-supervised learning.} There are many interesting theory results for self-supervised learning learning~\citep{wen2021toward, wen2022mechanism, luo2022one, haochen2021provable, arora2019theoretical, tosh2021contrastive, lee2021predicting, zimmermann2021contrastive}. For example, 
\citet{haochen2021provable,tan2023contrastive} show that SimCLR is essentially computing spectral graph clustering on the unlabeled dataset. 
\citet{liu2022rectified} gives a better rectified flow algorithm with elegant theoretical guarantees for improving the diffusion models.

\textbf{Application of category theory.} Category theory has been applied to many research areas~\citep{fong2018seven}, 
including physics~\citep{marquis2008geometrical, kus2019category},
design~\citep{censi2015mathematical}, and machine learning~\citep{shiebler2021category, mahadevan2022categoroids, mahadevan2022unifying}. The most relevant paper on category theory might be~\cite{bradley2022enriched}, where they provide an enriched category theory of natural language.

\section{Preliminaries}
Category theory is used in almost all areas of mathematics. Here we only introduce the necessary notions for understanding the results of our paper, and skip many important details (e.g., universe and diagram). Curious readers may check 
\citet{mac2013categories,riehl2017category, adamek1990abstract} for more comprehensive introductions. 

A category $\cat$ has a set of objects $\ob(\cat)$, and 
a set of morphisms $\Hom_\cat(X,Y)$ from $X$ to $Y$ 
for every $X, Y\in \ob(\cat)$.
Given $f\in \Hom_\cat(X,Y), g\in \Hom_\cat(Y,Z)$, we define their composition as $g\circ f\in  \Hom_\cat(X,Z)$. Notice that $\circ$ is associative, i.e.,  $(h\circ g) \circ f=h\circ (g\circ f)$. 
For every $X\in \ob(\cat)$, there exists a unique identity morphism $\id_X \in \hom_\cat(X,X)$. 
A morphism $f: X\rightarrow Y$   is an isomorphism if there exists $g: X\leftarrow Y$ such that $f\circ g=\id_Y$ and $g\circ f=\id_X$. 
In this case, we say $X$ and $Y$ are isomorphic and write $X\simeq Y$. 

Given a category $\cat$, we define its opposite $\cat^{\op}$ by setting $\ob(\cat^{\op})=\ob(\cat)$ and  $\hom_\cop(X,Y)=\hom_\cat(Y,X)$. Moreover, given $f\in \Hom_\cop(X,Y), g\in \Hom_\cop(Y,Z)$, the new composition is $g
\mathrel{\overset{\makebox[0pt]{\mbox{\normalfont\tiny\sffamily op}}}{\circ}}  f= f\circ g\in  \Hom_\cop(X,Z)$.

We define $\set$ to be the category of sets, 
where the objects are sets, and $\hom_\set(X,Y)$ is the set of all functions with domain $X$ and codomain $Y$. 
Notice that we ignore the subtleties about the universe for better presentation, so here just assume that $\set$ does not contain strange objects like a set containing all sets.

A functor is a structure-preserving map between categories. Given two
categories $\cat$ and $\cat'$, a functor
$
F:\cat\rightarrow \cat'
$
assigns to each object $X\in \ob(\cat)$ an object
$F(X)\in \ob(\cat')$, and to each morphism
$f\in \hom_\cat(X,Y)$ a morphism
$
F(f)\in \hom_{\cat'}(F(X),F(Y)).
$
It preserves identities and composition: for every $X\in \ob(\cat)$,
$
F(\id_X)=\id_{F(X)},
$
and for every pair of composable morphisms
$
f:X\rightarrow Y, g:Y\rightarrow Z,
$
we have
$
F(g\circ f)=F(g)\circ F(f).
$ A functor $F:\cat\rightarrow \cat'$ is faithful if, for every
$X,Y\in \ob(\cat)$, the induced map
$
F_{X,Y}:\hom_\cat(X,Y)\rightarrow
\hom_{\cat'}(F(X),F(Y))
$
is injective. It is full if $F_{X,Y}$ is surjective for every
$X,Y\in \ob(\cat)$. It is fully faithful if it is both full and
faithful. In this paper, we call $F$ an embedding functor if it is
faithful and injective on objects. We call $F$ a full embedding if it is
full, faithful, and injective on objects.

The morphisms between functors are called natural transformations. Given
two functors
$
F_1,F_2:\cat\rightarrow \cat',
$
a natural transformation
$
\theta:F_1\Rightarrow F_2
$
assigns to every object $X\in \ob(\cat)$ a morphism
$
\theta_X:F_1(X)\rightarrow F_2(X)
$
in $\cat'$, such that for every morphism $f:X\rightarrow Y$ in $\cat$,
the following naturality condition holds:
$
\theta_Y\circ F_1(f)=F_2(f)\circ \theta_X.
$
We write $F_1\simeq F_2$ if there exists a natural isomorphism
$\theta:F_1\Rightarrow F_2$, i.e., a natural transformation whose
components $\theta_X$ are isomorphisms in $\cat'$.
A functor $F:\cat\rightarrow \cat'$ is an isomorphism of categories if
there exists a functor $G:\cat'\rightarrow \cat$ such that
$
G\circ F=\id_\cat,
F\circ G=\id_{\cat'}.
$
In this case, we say $\cat$ and $\cat'$ are isomorphic and write
$\cat\simeq \cat'$.

A category $\mathcal{A}$ is a subcategory of a category $\mathcal{B}$, if 
\begin{enumerate}[leftmargin=\enumerateLeftMargin, parsep=0cm, topsep=0cm]
\item  $\ob(\mathcal{A})\subseteq 
\ob(\mathcal{B})$, 
\item for each $A, A'\in \ob(\mathcal{A}), \hom_{\mathcal{A}}(A, A')\subseteq 
\hom_{\mathcal{B}}(A, A')$, 
\item  for each $\mathcal{A}$-object $A$, the $\mathcal{B}$-identity on $A$ is the $\mathcal{A}$-identity on $A$; 
\item  the composition law in $\mathcal{A}$ is the restriction of the composition law in $\mathcal{B}$ to the morphisms of $\mathcal{A}$. 
\end{enumerate}
Moreover, $\mathcal{A}$ is a full subcategory of $\mathcal{B}$, if it is a subcategory of $\mathcal{B}$ and for each for each $A, A'\in \ob(\mathcal{A}), \hom_{\mathcal{A}}(A, A')=
\hom_{\mathcal{B}}(A, A')$. 
Every subcategory $\mathcal{A}$ of category $\mathcal{B}$ naturally defines an inclusion functor $E: \mathcal{A} \hookrightarrow \mathcal{B}
$, which is an embedding. For such embeddings, we have the following lemma. 

\begin{mylem}[\citet{adamek1990abstract}]
\label{lem:feature-aligned}
A functor $F:\cat\rightarrow \mathcal{B}$ is a full embedding if and only if there exists a full subcategory $\mathcal{A}$ of $\mathcal{B}$ with inclusion functor $E:\mathcal{A}\hookrightarrow \mathcal{B}$ and an isomorphism $G:\cat\rightarrow \mathcal{A}$ with $F=E\circ G$. 
\end{mylem}

\subsection{Contrastive methods}
Contrastive methods (like SimCLR~\citep{chen2020simple} and MoCo~\citep{he2020momentum}) take a query point $q$, and one positive sample $p_1$ to $q$, as well as $N-1$ other objects $\{p_i\}_{i=2}^N$. Then they aim to find a function $f$ to minimize the following loss function:
\[
\mathcal{L}(q, p_1, \{p_i\}_{i=2}^N)= -\log \frac{\exp(\text{sim}(f(q),f(p_1))/\tau)}{
\sum_{i=1}^N \exp(\text{sim}(f(q),f(p_i))/\tau)
}
\]
The actual loss function takes summation over different $q$, and $\tau$ is a temperature hyperparameter. $\text{sim}(Z_i, Z_j)=-\|Z_i-Z_j\|^2/2$, so $\exp(\text{sim}(f(q),f(p_1))/\tau)$ can be seen as a Gaussian kernel for $f(q)-f(p_1)$, whose normalizing constant $\frac{1}{\sqrt{2\pi \tau}}$ is canceled out in the fraction.

\subsection{Language Modeling}
Given a sentence \( S = (w_1, w_2, \dots, w_N) \), where each \( w_i \) is a token in the sentence, language modeling aims to learn a function \( f \) that estimates the probability distribution over the next token based on the preceding context. The objective is to minimize the negative log-likelihood of the observed sequence of tokens. This is typically expressed through the following loss function:

\[
\mathcal{L}(S) = - \sum_{i=2}^N \log f(w_i \mid w_1, w_2, \dots, w_{i-1})
\]

Here, \( f(w_i \mid w_1, w_2, \dots, w_{i-1}) \) represents the conditional probability that the next token in the sequence is \( w_i \), given the previous tokens \( w_1, w_2, \dots, w_{i-1} \). The function \( f \) takes the preceding tokens as input and outputs a probability distribution over the possible next tokens. The goal is to maximize the likelihood that the model assigns to the actual sequence of tokens observed in the sentence.

\subsection{Reproducing Kernel Hilbert Space (RKHS)}
Given two objects $X, Y\in \cat$, consider a feature map $f: \cat\rightarrow \mathcal{H}$, where the feature space $\mathcal{H}$ is usually much larger than $\cat$. We may define a kernel $k$ that measures the similarity of $X$ and $Y$ as $k(x, y)\triangleq \langle f(x), f(y)\rangle_\mathcal{H}$, i.e., the inner product between the two object after mapping them to the feature space. For any vector $T\in \mathcal{H}$, it also corresponds to a function $T(\cdot): \cat\rightarrow \mathbb{R}$, defined as $T(x)=\langle T, f(x)\rangle_\mathcal{H}$. Specifically,
$f(y)$ as a vector in $\mathcal{H}$ also represents the function $k(\cdot, y): \cat\rightarrow \mathbb{R}$, because for any $x\in \cat$, we have $k(x, y)=\langle f(x), f(y)\rangle_\mathcal{H}$. Formally, we have:
\begin{mydef}[Reproducing kernel Hilbert space]
	\label{def:rkhs}
	Let $\mathcal{H}$ be a Hilbert space of $\mathbb{R}$-valued functions defined on a non-empty set $\cat$. A function $k:\cat\times \cat \rightarrow \mathbb{R}$ is called a reproducing kernel of $\mathcal{H}$, and $\mathcal{H}$ is a
	reproducing kernel Hilbert space (RKHS), if $k$ satisfies
	\begin{itemize}[leftmargin=0.3cm, parsep=0cm, topsep=0cm]
		\item $\forall x\in \cat, k(\cdot, x)\in \mathcal{H}$,
		\item $\forall x\in \cat, \forall f\in \mathcal{H}, \langle f, k(\cdot, x)\rangle_{\mathcal{H}}=f(x).$
	\end{itemize}    
\end{mydef}

\section{Categorical Framework of Supervised Learning}
\label{sec:supervised}
In supervised learning, we have a population distribution $D=(D_X, D_Y)$, representing the ground truth distribution of the input and output data points. The training set $(X_{\mathrm{train}}, Y_{\mathrm{train}})$ and test set
$(X_{\mathrm{test}}, Y_{\mathrm{test}})$ are  uniformly sampled from  $(D_X, D_Y)$. We hope to learn a function $f: X\rightarrow Y$ so that $f(x)$ accurately predicts the  label $x\in X$. We also define a loss function $L(f,x, y)$ to measure the distance between the prediction $f(x)$ and the correct label, which is hopefully close to $0$. The loss function on a dataset $(X,Y)$ is  $L(f,X,Y)\triangleq \E_{(x,y)\sim (X,Y) }L(f,x,y)$. 
The task of supervised learning, is to minimize the population loss $L_{\mathrm{population}}\triangleq 
\E_{(X,Y)\sim (D_X,D_Y)} L(f,X,Y)$, with access of the training set $(X_{\mathrm{train}}, Y_{\mathrm{train}})$.

Using the language of category theory, we have two categories $\cx$ and $\cy$, with objects $x, y$ as input and output data points, respectively. 
To avoid confusion, below we switch notation from $x, y$ to $X, Y$ to represent the objects, as later we will not use $D_X, D_Y$ any more. 
The population distribution can be seen as a functor 
$F$ from $\cx$ to $\cy$, representing the correct label $Y$ given the input $X$. Due to the inherent noise in the real world, there may not always exist the unique correct label for each input $X$. In other words, the Bayesian optimal solution does not give zero population loss. We will discuss this issue in Section~\ref{subsec:bayesian}, and for now let us simply assume that the unique correct labels always exist. 

With this formulation, training/test set can be seen as samples over objects in $\cx$ with correct labels in $\cy$. 
Therefore, supervised learning investigates the following question: 
can we learn a functor $F$ with samples of $X\in \cx$ and $F(X)~\in~\cy$? 

Consider the special case that both categories are discrete, meaning that the only morphisms existed are identity morphisms like $\id_X\in \hom_\cx(X,X)$ and $\id_Y\in \hom_\cy(Y,Y)$. In other words, both categories are sets without any morphisms between different objects. In this case, 
learning $F$ is impossible, because it maps a set to another set without any prior knowledge. The no free lunch theorem tells us that unless we have sample size larger than half of the set size, the functor we learned will have constant generalization error with constant probability.

Generalization theory deals with this problem by assuming 
that $F$ is in a predefined hypothesis class $\mathcal{H}$, or close to a function in $\mathcal{H}$. 
In category theory, we do not add restrictions to the functors, but to the structure of the categories instead. For example, 
if we know $\cx$ has a linear structure, and $F$ preserves the structure, it is possible to learn $F$ with a few samples.

This formulation seems useless, 
as it does not even characterize the loss function, which is crucial in supervised learning. 
This is because
\textbf{category theory takes the bird's-eye view.} Our categorical framework will not replace the classical framework, or generate better supervised learning algorithms. Instead, it treats the existing supervised learning algorithms a subroutine or a building block, which can be incorporated into a bigger picture. Therefore, it does not care about the loss functions or optimization process, which are treated as the implementation details. Instead, it focuses on the structure of the categories and functors, and tries to understand whether certain functors are learnable or not. It also investigates whether the ability of mastering at one or more tasks can be generalized to other tasks.

\section{Categorical Framework of Self-supervised Learning}
\subsection{Pretext tasks and relationships}
\label{subsec:pretext-task}
In self-supervised learning, we have a population distribution $D$  of data points without labels, and try to extract useful information from $D$ by setting up pretext tasks. Different pretext tasks give different kinds of relationships between two data points, for example:

\begin{itemize}[leftmargin=0.4cm]
	\item Contrastive methods. 
	Contrastive methods like SimCLR~\citep{chen2020simple} modify data point $X\in D$ to get two semantically similar copies $X', X''\in D$ as the positive pair, and pick another different $Y\in D$ as the negative sample. The pretext task says $X', X''$ should be close and both of them should be far away from $Y$. 
	It has been proved that this task is equivalent to learning similarities between data points~\citep{tan2023contrastive,haochen2021provable}.

	\item Language modeling. For every sentence (or sub-sentence) $X\in D$, the pretext task models the distribution of the next token following $X$. By iteratively applying this task, we are learning the probability of a sentence $Y\in D$ is successive to sentence $X$, where $X$ is a prefix of $Y$. 
\end{itemize}

We remark that the relationships defined by the pretext tasks are not necessarily consistent, therefore cannot be directly converted into morphisms of a category. 
A natural question is:

\begin{center}
	\fbox{
		\parbox{\questionwidth}{
			\textbf{
	Can we find a category that approximates the relationships introduced by pretext tasks?
	}}}
\end{center}

To answer this question, we should first 
define how a category is induced  by a foundation model.

\subsection{Foundation model induced category}
A foundation model maps an input to an internal representation, and
subsequent prediction heads or task-specific readouts compute outputs
from this representation. In this paper, we write
$
f:D\rightarrow \mathcal{H}
$
for the representation map of a trained foundation model. When the
pretext task provides additional fixed forward operations on
representation states, such as similarity computations, prediction
heads, or autoregressive state transitions, we regard these operations as
part of the model's representation interface.

In general, different inputs may
have the same representation, and such inputs cannot be distinguished by
any readout that only has access to the representation. Therefore, the
category induced by a foundation model should be defined on observable
representation states, or equivalently on inputs modulo representation
equivalence.

\begin{mydef}[Foundation model induced category]
	\label{def:induced}
	Let $f:D\rightarrow \mathcal{H}$ be a foundation model representation.
	Define an equivalence relation on $D$ by
	\[
	X\sim_f X' \quad \Longleftrightarrow \quad f(X)=f(X').
	\]
	Let $\mathcal{H}_f\triangleq f(D)\subseteq \mathcal{H}$ be the set of
	reachable representation states. We say $f$ induces a category $\catf$
	if there exists a data-oblivious readout
	\[
	k_f:\mathcal{H}_f\times \mathcal{H}_f\rightarrow \set
	\]
	and a family of composition maps
	\[
	\circ_{X,Y,Z}:
	k_f(f(Y),f(Z))\times k_f(f(X),f(Y))
	\rightarrow k_f(f(X),f(Z)),
	\]
	such that:
	\begin{itemize}
		\item $\ob(\catf)\triangleq D/{\sim_f}$. By abuse of notation, we
		write an object as $X$, and identify it with its reachable
		representation state $f(X)\in \mathcal{H}_f$.
		
		\item For any $X,Y\in \ob(\catf)$,
		\[
		\hom_{\catf}(X,Y)
		\triangleq
		k_f(f(X),f(Y)).
		\]
		This definition is well-defined because $f$ is constant on each
		equivalence class.
		
		\item Composable: for any $X,Y,Z\in \ob(\catf)$,
		$g\in \hom_{\catf}(X,Y)$, and
		$h\in \hom_{\catf}(Y,Z)$, the composition
		$h\circ g$ is an element of $\hom_{\catf}(X,Z)$.
		
		\item Associative: for any composable morphisms $g,h,u$,
		\[
		(u\circ h)\circ g = u\circ (h\circ g).
		\]
		
		\item Unique identity: for every $X\in \ob(\catf)$, there exists
		a unique identity morphism
		\[
		\id_X\in \hom_{\catf}(X,X)
		\]
		such that
		\[
		g\circ \id_X=g
		\quad\text{and}\quad
		\id_Y\circ g=g
		\]
		for every $g\in \hom_{\catf}(X,Y)$.
	\end{itemize}
\end{mydef}

Here, data-oblivious means that the readout $k_f$ and the composition
law are specified before seeing the empirical data distribution. The
subscript $f$ indicates that the readout may depend on the trained
representation and on the pretext-task structure associated with the
model. 
Thus the induced category only contains distinctions that are observable
from the representation. If two inputs are collapsed by $f$, then they
are identified as the same object in $\catf$.

A useful general way to construct such readouts is through a
pretext-task-induced action on the representation space. Let $M_f$ be a
monoid with unit $e$, determined by the fixed interface of the trained
model, and suppose there is a right action
\[
\alpha_f:\mathcal{H}_f\times M_f\rightarrow \mathcal{H}_f
\]
satisfying
\[
\alpha_f(h,e)=h,
\qquad
\alpha_f(\alpha_f(h,m),n)=\alpha_f(h,mn).
\]
The subscript $f$ emphasizes that the action may use the trained model's
fixed operations on representation states.

Given such an action, define
\[
k_f(f(X),f(Y))
\triangleq
\{m\in M_f:\alpha_f(f(X),m)=f(Y)\}.
\]
For
\[
m\in k_f(f(X),f(Y)),
\qquad
n\in k_f(f(Y),f(Z)),
\]
define
\[
n\circ m\triangleq mn.
\]
Then
\[
\alpha_f(f(X),mn)
=
\alpha_f(\alpha_f(f(X),m),n)
=
f(Z),
\]
so $mn\in k_f(f(X),f(Z))$. The identity morphism at $X$ is the monoid
unit $e$. The action law implies composability and associativity, and
the unit law gives the identity morphisms. Hence any such action induces
a category on the reachable representation states.

This construction should be viewed as a representation-level version of
the relationships specified by the pretext task. For contrastive
learning, the relevant action is translation in feature space. For
language modeling, the relevant action is the autoregressive transition
obtained by appending tokens to a state.

As we will see shortly, category theory provides many great tools to analyze our model. However, we still need to show Definition~\ref{def:induced} is not vacuous, as demonstrated by the following two theorems. 

\begin{mythm}
	\label{thm:simCLR}
	Contrastive methods induce categories. 
\end{mythm}
\begin{proof}
	For contrastive methods, the foundation model maps each input $X$ to a
	feature vector $f(X)\in \mathcal{H}$. Let
	\[
	\mathcal{H}_f=f(D)\subseteq \mathcal{H}
	\]
	be the set of reachable feature states. As in
	Definition~\ref{def:induced}, we regard the objects of the induced
	category as inputs modulo representation equivalence, and write them as
	$X,Y,Z$.
	
	Let
	\[
	M_f=(\mathcal{H},+)
	\]
	be the additive monoid of the feature space. It acts on reachable feature
	states by translation:
	\[
	\alpha_f(z,v)=z+v.
	\]
	Define
	\[
	k_f(f(X),f(Y))
	\triangleq
	\{v\in \mathcal{H}:\alpha_f(f(X),v)=f(Y)\}.
	\]
	Equivalently,
	\[
	\hom_{\catf}(X,Y)
	=
	k_f(f(X),f(Y))
	=
	\{f(Y)-f(X)\}.
	\]
	This is well-defined on objects of $D/{\sim_f}$ because $f$ is constant
	on each equivalence class.
	
	For
	\[
	g=f(Y)-f(X)\in \hom_{\catf}(X,Y),
	\qquad
	h=f(Z)-f(Y)\in \hom_{\catf}(Y,Z),
	\]
	define the composition by
	\[
	h\circ g \triangleq g+h.
	\]
	Then
	\[
	h\circ g
	=
	\bigl(f(Y)-f(X)\bigr)+\bigl(f(Z)-f(Y)\bigr)
	=
	f(Z)-f(X)
	\in \hom_{\catf}(X,Z).
	\]
	Therefore composition is well-defined.
	
	We now verify the category axioms.
	\begin{itemize}
		\item Composable. The calculation above shows that if
		$g\in \hom_{\catf}(X,Y)$ and
		$h\in \hom_{\catf}(Y,Z)$, then
		$h\circ g\in \hom_{\catf}(X,Z)$.
		
		\item Associative. Composition is vector addition in
		$\mathcal{H}$, and vector addition is associative. Hence for any
		composable morphisms $g,h,u$,
		\[
		(u\circ h)\circ g = u\circ (h\circ g).
		\]
		
		\item Unique identity. For every object $X$, the identity morphism is
		the zero vector
		\[
		\id_X=\mathbf{0}\in \hom_{\catf}(X,X),
		\]
		because
		\[
		\alpha_f(f(X),\mathbf{0})=f(X).
		\]
		For any $g\in \hom_{\catf}(X,Y)$,
		\[
		g\circ \id_X = g,
		\qquad
		\id_Y\circ g = g.
		\]
		Since
		\[
		\hom_{\catf}(X,X)=\{f(X)-f(X)\}=\{\mathbf{0}\},
		\]
		this identity morphism is unique.
	\end{itemize}
	
	Therefore contrastive methods with foundation model $f$ induce a
	category $\catf$ on reachable feature states.
\end{proof}

\begin{mythm}Language modeling induces categories.
	\label{thm:gpt}
\end{mythm}
\begin{proof}
	Let $V$ be the token vocabulary and let $V^\ast$ be the free monoid of
	finite token strings under concatenation. The unit element is the empty
	string $\epsilon$.
	
	For language modeling, we take the foundation model representation
	$f(X)\in \mathcal{H}$ to be an autoregressive state of the model after
	reading the context $X$. Such a state may be, for example, a sufficient
	hidden state or a key-value cache state. The important property is that
	the model can be advanced forward after appending tokens.
	
	Let
	\[
	M_f=V^\ast.
	\]
	The fixed autoregressive transition of the trained language model gives
	a right action
	\[
	\alpha_f:\mathcal{H}_f\times V^\ast\rightarrow \mathcal{H}_f
	\]
	satisfying
	\[
	\alpha_f(f(X),u)=f(Xu)
	\]
	for every context $X$ and every token string $u\in V^\ast$.
	
	Indeed,
	\[
	\alpha_f(f(X),\epsilon)=f(X),
	\]
	and for any $u,v\in V^\ast$,
	\[
	\alpha_f(\alpha_f(f(X),u),v)
	=
	\alpha_f(f(Xu),v)
	=
	f(Xuv)
	=
	\alpha_f(f(X),uv).
	\]
	Therefore $\alpha_f$ is a right action of $V^\ast$ on reachable
	autoregressive states.
	
	For each object $Y$, denote by $P_Y$ the distribution of all future
	continuations generated after the model is in state $f(Y)$. This
	distribution is determined only by the state $f(Y)$ and the model's
	fixed next-token readout, not by how the state $f(Y)$ was reached.
	
	More explicitly, let
	\[
	o_f:\mathcal{H}_f\rightarrow \Delta(V)
	\]
	be the next-token readout of the trained language model. Then $P_Y$ is
	the autoregressive continuation law starting from $f(Y)$. For a finite
	continuation
	\[
	w=w_1w_2\cdots w_m\in V^\ast,
	\]
	its prefix probability is
	\[
	P_Y(w_1,\ldots,w_m)
	=
	\prod_{i=1}^m
	o_f\bigl(\alpha_f(f(Y),w_{<i})\bigr)(w_i),
	\]
	where $w_{<i}=w_1\cdots w_{i-1}$. If one includes an end-of-sequence
	token, this gives a probability distribution over finite continuations;
	otherwise it defines the usual cylinder probabilities over infinite
	continuations.
	
	Now define the readout $k_f$ by
	\[
	k_f(f(X),f(Y))
	\triangleq
	\begin{cases}
		\{(X,Y,P_Y)\},
		& \text{if there exists }u\in V^\ast
		\text{ such that }\alpha_f(f(X),u)=f(Y),\\
		\emptyset,
		& \text{otherwise.}
	\end{cases}
	\]
	Equivalently,
	\[
	\hom_{\catf}(X,Y)
	=
	\begin{cases}
		\{(X,Y,P_Y)\},
		& \text{if }Y\text{ is reachable from }X
		\text{ by the forward autoregressive transition},\\
		\emptyset,
		& \text{otherwise.}
	\end{cases}
	\]
	Thus a morphism from $X$ to $Y$ records the future-continuation
	distribution available after reaching $Y$. It does not record the
	probability of generating $Y$ from $X$.
	
	Given
	\[
	g=(X,Y,P_Y)\in \hom_{\catf}(X,Y),
	\qquad
	h=(Y,Z,P_Z)\in \hom_{\catf}(Y,Z),
	\]
	define the composition function by
	\[
	h\circ g \triangleq (X,Z,P_Z).
	\]
	
	We verify that this is well-defined. Since
	$g\in \hom_{\catf}(X,Y)$, there exists $u\in V^\ast$ such that
	\[
	\alpha_f(f(X),u)=f(Y).
	\]
	Since $h\in \hom_{\catf}(Y,Z)$, there exists $v\in V^\ast$ such that
	\[
	\alpha_f(f(Y),v)=f(Z).
	\]
	Therefore, by the action law,
	\[
	\alpha_f(f(X),uv)
	=
	\alpha_f(\alpha_f(f(X),u),v)
	=
	\alpha_f(f(Y),v)
	=
	f(Z).
	\]
	Hence there exists a token string $uv\in V^\ast$ such that
	\[
	\alpha_f(f(X),uv)=f(Z).
	\]
	By the definition of $k_f$, this implies
	\[
	(X,Z,P_Z)\in \hom_{\catf}(X,Z).
	\]
	Therefore composition is well-defined.
	
	We now verify the category axioms.
	\begin{itemize}
		\item Composable. The argument above shows that if
		$g\in \hom_{\catf}(X,Y)$ and
		$h\in \hom_{\catf}(Y,Z)$, then
		$h\circ g\in \hom_{\catf}(X,Z)$.
		
		\item Associative. Consider composable morphisms
		\[
		g=(X,Y,P_Y),\qquad
		h=(Y,Z,P_Z),\qquad
		r=(Z,W,P_W).
		\]
		Then
		\[
		(r\circ h)\circ g=(X,W,P_W),
		\]
		and
		\[
		r\circ(h\circ g)=(X,W,P_W).
		\]
		Hence
		\[
		(r\circ h)\circ g=r\circ(h\circ g).
		\]
		
		\item Unique identity. For every object $X$, since
		\[
		\alpha_f(f(X),\epsilon)=f(X),
		\]
		we have
		\[
		\hom_{\catf}(X,X)=\{(X,X,P_X)\}.
		\]
		Define
		\[
		\id_X\triangleq (X,X,P_X).
		\]
		For any
		\[
		g=(X,Y,P_Y)\in \hom_{\catf}(X,Y),
		\]
		we have
		\[
		g\circ \id_X=(X,Y,P_Y)=g.
		\]
		Similarly,
		\[
		\id_Y\circ g=(X,Y,P_Y)=g.
		\]
		Since $\hom_{\catf}(X,X)$ is a singleton, this identity morphism is
		unique.
	\end{itemize}
	
	Therefore language modeling with foundation model $f$ induces a category
	$\catf$ on reachable autoregressive states.
\end{proof}

\subsection{Category approximates pretext task}

For the two constructions above, once the corresponding readout $k_f$
and forward interface are fixed, any choice of model parameters gives a
well-defined category, although the resulting category may be degenerate
or poorly aligned with the intended pretext-task structure.
In particular, even when the weights of $f$ are random, there still exists a possibly strange category for $f$. 
Therefore, if we name $f$ with its weight $w_t$ as $f_t$ for the step $t$ during the training process, then in every step~$t$ we will get a difference induced category $\cat_{f_t}$. In other words, the training process of $f_t$ also corresponds to the training process of categories $\cat_{f_t}$. Hence, as the answer to the question
in Section~\ref{subsec:pretext-task}, 
we can find a category that approximates the relationships introduced by the pretext task through the learning process.

A common question might arise: ``Why approximate pretext tasks using categories in the first place?'' 
However, as explained earlier,
since foundation models \textbf{always} correspond to well defined categories in the two scenarios, the right question should be ``how is it possible that we can always learn a good category when pretext tasks are not consistent?''

The answer lies in the nature of the learning process: it ensures that $\cat_{f_t}$
is always a well-defined category, even if the morphisms do not perfectly align with the pretext task
For example, as shown in \cite{tan2023contrastive},
the model learns the spectral clustering solution of the similarity graph derived from the pretext task. However, the final solution may establish pairwise similarities between all objects, which may not precisely match the original pretext task.

Consequently, whether the pretext task is consistent or not does not impact our analysis, since the model ultimately learns a consistent category that best approximates the relationships specified by the pretext task. Given any pretext task dataset and a learning algorithm, such a category always exists\footnote{There may even exist multiple different categories due to the randomness of the training process, which can be analyzed similarly and we omit the discussion here.}. For the rest of the paper, we will refer to the category that best matches the pretext task as the ``ideal category $\cat$'' and the corresponding foundation model as the ``ideal foundation model'' for $\cat$.

\subsection{Yoneda Lemma}
By the previous definitions, 
when the pretext task is given, even if the training process is optimal, an ideal foundation model $f$ is all we expect to get, which corresponds to an ideal category $\cat$. We raise the following question,

\begin{center}
	\fbox{
		\parbox{4.2in}{
			\textbf{Is $f$ guaranteed to have some property, if it is ideal with a category $\cat$?
	}}}
\end{center}

This is the main motivation of our paper. Below we first introduce an important notion called Yoneda embedding. 

\begin{mydef}[Yoneda embedding $h_\cat$]
Given any $X\in \cat$, 
$h_\cat(X)\triangleq \hom_\cat(X,\cdot)$, which takes input $Y\in \cat$ and outputs $\hom_\cat(X,Y)$.	\footnote{In the notation of \citet{kashiwara2006categories}, this functor is denoted by $k_C$ rather than $h_C$. We use $h_\cat$ here to avoid conflict with the kernel notation $k$.}
\end{mydef}

Define $\cat^\vee$ as the category of functors from $\cat^{op}$ to $\set^{op}$. Then $h_\cat:\cat \rightarrow \cat^\vee$ is well-defined. Notice that we can define $\hom_\cat(\cdot,X)$ similarly, which gives another kind of Yoneda embedding.

How do we represent $\cat^\vee$ empirically? To answer this question, we should first learn arguably the most important theorem in category theory, as follows.

\begin{mylem}[Yoneda lemma] 
	\label{lem:yoneda}
	For $A\in \cat^{\vee}$, and $X\in \cat$, $\hom_{\cat^\vee}(A, h_\cat(X))\simeq
	A(X). $
\end{mylem}

In category theory, when two objects (or functors) are isomorphic, we treat them as equal. The technical details for making them equal to each other is omitted, as we observe that this is not a problem for modern networks empirically.  For example, if there exists an isomorphism between two objects $X$ and $Y$, so that $\phi(X)=Y$, empirically a deep neural network can easily fit this $\phi$. More generally, techniques like rectified flow~\citep{liu2022rectified} can fit very complicated isomorphisms between random noise and images, using neural network with ODE.

Let $\cat=\cat_f$ be the category induced by the foundation model $f$
and the readout $k_f$. By construction, the objects of $\cat$ are
observable representation states, written as $X,Y$, and
\[
\hom_\cat(X,Y)=k_f(f(X),f(Y)).
\]
Once $\cat$ is constructed, it has the canonical Yoneda embedding
\[
h_\cat:\cat\rightarrow \cat^\vee,
\qquad
h_\cat(X)=\hom_\cat(X,\cdot).
\]

Strictly speaking, the feature vector $f(X)\in\mathcal{H}_f$ is not
itself a functor in $\cat^\vee$. Rather, since objects of $\cat$ are
identified with reachable representation states, there is a canonical
realization map
\[
\Phi_f:\mathcal{H}_f\rightarrow \cat^\vee
\]
defined by
\[
\Phi_f(f(X))\triangleq h_\cat(X).
\]
This is well-defined because objects are taken modulo representation
equivalence: if $f(X)=f(X')$, then $X$ and $X'$ denote the same object of
$\cat$.

Under this realization, $f(X)$ is interpreted as a code for the
representable functor $h_\cat(X)$. Therefore,
\[
k_f(f(X),f(Y))
=
\hom_\cat(X,Y)
=
h_\cat(X)(Y).
\]
By the Yoneda lemma, we also have
\[
h_\cat(X)(Y)
\simeq
\hom_{\cat^\vee}(h_\cat(X),h_\cat(Y)).
\]
Equivalently,
\[
k_f(f(X),f(Y))
\simeq
\hom_{\cat^\vee}\bigl(\Phi_f(f(X)),\Phi_f(f(Y))\bigr).
\]
From now on, when no confusion arises, we suppress the realization map
$\Phi_f$ from the notation and identify each reachable representation
$f(X)$ with its realization
\[
\Phi_f(f(X))=h_\cat(X)
\]
in $\cat^\vee$. Thus expressions such as $f(X)\in\cat^\vee$,
$T\simeq f(P)$, or $k_f(T,f(X))$ should be understood after applying this
canonical realization.

Equivalently, after this point we use $h_\cat$ as the Yoneda-realized
form of the ideal foundation model for the category $\cat$. That is,
when we refer to $h_\cat$ as a foundation model, we mean the realized
representation
\[
h_\cat \simeq \Phi_f\circ f,
\]
rather than the raw feature map $f:D\rightarrow \mathcal H$ itself. This
is only a notational convention: the actual logical direction remains
\[
(f,k_f)\Longrightarrow \cat_f \Longrightarrow h_\cat,
\]
and the feature vector $f(X)\in\mathcal{H}_f$ is not literally a functor
before this realization is applied.

Under this convention, $f(X)$ encodes, and will be identified with, the
representable functor $h_\cat(X)=\hom_\cat(X,\cdot),$ which takes input $Y\in\cat$ and outputs $\hom_\cat(X,Y)$.

The above interpretation may look redundant at first glance, because
after applying the convention above, one may simply write $f(X)$ as if it
were $h_\cat(X)$. The reason for first introducing $\Phi_f$ is to make
the logical direction explicit. We do not start from a given category
$\cat$ and then choose a representation of its objects. Instead, we start
from the foundation model $f$ and the readout $k_f$, construct the
category $\cat_f$, and only then interpret the reachable representation
$f(X)$ as the representable functor $h_\cat(X)$.

\subsection{Downstream tasks}
After the model is trained, we apply it to the downstream tasks with two different approaches: prompt tuning and fine tuning. In this section, we investigate the power of both approaches, with different outcomes. We first define what we mean by solving a downstream task.

\begin{mydef}[Task]
	A task $T$ is a functor in $\cat^\vee$. 	
\end{mydef}

\begin{mydef}[Task solving]
	We say the model solves a task $T$, if the downstream functor induced by the model is isomorphic to $T$. Equivalently, for any input $X\in\cat$, the model outputs a solution that is isomorphic to $T(X)$.
\end{mydef}

Given a task defined as a functor $T\in \cat^\vee$, prompt tuning means we freeze the parameters of the model, and only use a task specific prompt $P$ (usually in text or image), followed by the actual input of the task $X$, to get the output $T(X)$. Therefore, the prompt $P$ and input $X$ are the only two inputs to the model. By Lemma~\ref{lem:yoneda}, if we directly send $T$ and $h_\cat(X)$ to $k_f$, we have
\begin{equation}
	\label{eqn:prompt}
	T(X)\simeq \hom_{\cat^\vee}(T, h_\cat(X))
	\triangleq k_f(T, f(X)).
\end{equation}

That is, $k_f$ can accurately compute $T(X)$ from these two representations. However, in prompt tuning, the task representation must itself be realized by a prompt. Therefore, instead of sending $T$ directly to $k_f$, we send the representation $h_\cat(P)$ induced by the prompt $P$. This brings up another important definition in category theory. 

\begin{mydef}[Representable functor]
	A functor $T\in \cat^\vee$ is representable if there is an isomorphism $h_\cat(X)\simeq T$ for some $X\in \cat$. Such an object $X$ is called a representative of $T$.  
\end{mydef}

Based on this definition, we have the following characterization of prompt tuning. 
\begin{mythm}[Power on prompt tuning]
	\label{thm:prompt-tuning}
$f$ can solve $T$ with prompt tuning, if and only if task $T\in \cat^\vee$ is representable. When $T$ is representable, the optimal prompt is the representative of $T$. 
\end{mythm}

\begin{proof}
	When the task $T$ is representable, let $P\in\cat$ be a representative of $T$, so that $T\simeq h_\cat(P)$. Then by (\ref{eqn:prompt}),
	\[
	T(X)\simeq \hom_{\cat^\vee}(T,h_\cat(X))
	\simeq \hom_{\cat^\vee}(h_\cat(P),h_\cat(X))
	\triangleq k_f(f(P),f(X)).
	\]
	Hence the model solves $T$ with prompt $P$.
	
	On the other hand, suppose the model solves $T$ with prompt tuning using a prompt $P$. Then for any input $X\in\cat$, the model output is
	\[
	k_f(f(P),f(X))
	\triangleq \hom_{\cat^\vee}(h_\cat(P),h_\cat(X))
	\simeq h_\cat(P)(X).
	\]
	Therefore, the downstream functor induced by the prompt $P$ is exactly $h_\cat(P)$. Since the model solves $T$, this functor must be isomorphic to $T$. Hence $T$ is representable.
\end{proof}

\textbf{Remark.}
There are some interesting results on continuous prompts~\citep{liu2021gpt, li2021prefix}, which directly tunes the prompts in the feature space of the neural network, therefore the resulting prompt is not a representation of any real words/sentences. 
By doing this, it is possible to obtain more power than the representable tasks, but the enhancement depends on the expressive power of the feature space. 
Below we provide a simple application of  Theorem~\ref{thm:prompt-tuning}.

\begin{mycor}[Pretext-only guarantee of rotation prediction]
	\label{cor:rotate}
	For the pretext task of predicting image rotations~\citep{gidaris2018unsupervised},
	the only structure guaranteed by the pretext objective itself is the
	four-fold rotational structure generated by
	\[
	0^\circ,\ 90^\circ,\ 180^\circ,\ 270^\circ.
	\]
	Consequently, without additional assumptions on the architecture,
	optimizer, data distribution, or implicit bias of the learned
	representation, prompt tuning is not guaranteed to solve content-level
	downstream tasks such as segmentation or classification.
\end{mycor}

\begin{proof}
	The rotation-prediction pretext task only requires the model to
	distinguish the four rotations of an image:
	\[
	0^\circ,\ 90^\circ,\ 180^\circ,\ 270^\circ.
	\]
	Therefore, the structure specified by the pretext task is the
	four-fold rotational action generated by these transformations.
	
	Since we make no assumptions on the implicit bias of the training
	algorithm or on additional information preserved by the representation,
	the only structure that is guaranteed to be learned from the pretext task
	alone is this rotational structure. A learned representation may in fact
	retain more information, such as semantic or spatial content. However,
	such information is not forced by the rotation-prediction objective and
	therefore cannot be used as a pretext-only guarantee.
	
	Under the pretext-guaranteed rotational structure, the representable
	functors
	\[
	h_\cat(X)=\hom_\cat(X,\cdot)
	\]
	can only express relationships determined by the four-fold rotation
	action. By Theorem~\ref{thm:prompt-tuning}, prompt tuning can solve a
	downstream task only when that task is representable in the induced
	category. Hence, from the rotation pretext task alone, we can only
	guarantee prompt-tuning solutions for tasks representable by this
	rotational structure.
\end{proof}

When applying Theorem~\ref{thm:prompt-tuning} to large language models, 
the statement is slightly different, because we have to send ``prompt + $X$'' to the model. 

Consider task $T$ that takes $X$ as the input, and outputs the distribution of sentences representing correct answers. For any prompt $P$, we can define a task $T_P$, such that $T_P(P+X)=T(X)$, and $T_P(X')=\emptyset$ if $X'$ is not started with $P$. 
This can be a bit confusing, but $T_P$ is defined by $T$, and not necessarily related to the prompt~$P$. 
For example, if $T$ is ``count the number of words in $X$'', but prompt $P=$``repeat the following sentence:'', then $T_P(\text{``repeat the following sentence: hello''})$ should give the answer ``$1$'', instead of ``hello''.

\begin{mythm}[Prompt tuning on LLM]
LLM $f$ can solve the task $T$ with prompt tuning, if and only if there exists a prompt $P$, such that
$T_P \simeq f(P)$. 
\end{mythm}

\begin{proof}
	If there exists a prompt $P$ such that $T_P \simeq f(P)$, then  we identify $f(P)$ with $h_\cat(P)$. Hence, by Yoneda Lemma,
	\[
	k_f(f(P), f(P+X))
	\simeq
	k_f(T_P, f(P+X))
	\simeq
	T_P(P+X)
	=
	T(X).
	\]
	
	On the other hand, suppose there exists a prompt $P$ such that
	\[
	T(X)=T_P(P+X)=k_f(f(P), f(P+X))
	\]
	for any input $X$. Then for every object of the form $P+X$, we have
	\[
	T_P(P+X)\simeq k_f(f(P), f(P+X))\simeq h_\cat(P)(P+X).
	\]
	If $X'$ does not start with $P$, then by definition
	\[
	T_P(X')=\emptyset,
	\]
	and by the prefix structure of the language-model category,
	\[
	h_\cat(P)(X')=\hom_\cat(P,X')=\emptyset.
	\]
	Therefore, $T_P$ and $h_\cat(P)$ agree on all objects. Since their functor structures are determined by the prefix structure of the language-model category, it follows that
	\[
	T_P \simeq h_\cat(P).
	\]
	this is equivalently written as
	\[
	T_P \simeq f(P).\qedhere
	\]
\end{proof}
The power of prompt tuning is limited, and can be characterized by representable functors. What about fine tuning? We have the following theorem:

\begin{mylem}[Yoneda Extension of Functors]
	\label{thm:yoneda-extension}
	Let $T\in \cat^\vee$. Then there exists an extension functor
	$
	h_\cat^{\dag}T:\cat^\vee \rightarrow \set
	$
	such that
	$
	h_\cat^{\dag}T \circ h_\cat \simeq T.
	$
	In fact, $h_\cat^{\dag}T$ is given by
	$
	h_\cat^{\dag}T(A)\triangleq \hom_{\cat^\vee}(T,A),
	 A\in \cat^\vee.
	$
\end{mylem}

The actual Yoneda extension of functors theorem is more general  (see e.g., Proposition 2.7.1 in \citet{kashiwara2006categories}), it can be applied to any category $\mathcal{A}$, but here we only use $\set$ to avoid unnecessary technical details of inductive limits. Using Lemma~\ref{thm:yoneda-extension}, we immediately get the following theorem.

\begin{mythm}[Power on fine tuning]
	\label{thm:fine-tuning}
	Ideally, given enough resources, $h_\cat$ can solve any task $T\in \cat^\vee$.
\end{mythm}

\begin{proof}
Applying Lemma~\ref{thm:yoneda-extension},
we know that with enough training data for the downstream task, computational power, and a fine tuning model $h$ that learns $h_\cat^ {\dag} T$ perfectly, we can concatenate $h$ with $f$, to solve the task $T$.
\end{proof}

The downstream tasks considered in Theorem~\ref{thm:fine-tuning} are based on the structure in $\cat$, not the data content in the dataset. As a result, the category defined by \cite{gidaris2018unsupervised} still has very simple group-structure, but with Theorem~\ref{thm:fine-tuning} it is possible to solve more diverse tasks. For example, we can map all the objects to the same output, which cannot be achieved with prompt tuning. Theorem~\ref{thm:fine-tuning} conveys the importance of pretext task, as more informative pretext tasks will create more structured category $\cat$, which further improves the power of fine tuning on $\cat^\vee$.

Even with very informative $\cat$, Theorem~\ref{thm:fine-tuning} only provides a raw upper bound with strong assumptions on the resources needed. By contrast, empirically the fine tuning step is usually done with a small network. Therefore, instead of treating it as a strong backup of what people are doing right now, it is more like exhibiting the future possibility. In other words, it implies that by transforming the objects to the corresponding feature space $\cat^\vee$, we have captured all the information encoded in $\cat$. 

For readers who are familiar with machine learning theory, Theorem~\ref{thm:fine-tuning} may look similar to the theory of over-parameterization at first glance. However, they are analyzing different steps of self-supervised learning.
Over-parameterization analyzes the pretraining step, saying that under certain assumptions, the optimization and generalization error will be very small for the pretext task, as long as the model is big enough and the learning rate is small enough. By contrast, Theorem~\ref{thm:fine-tuning} analyzes the fine tuning step after pretraining. Even if we have successfully pretrained a network on Imagenet with contrastive learning and zero generalization error, it remains unclear whether the model can be used for image segmentation as the downstream task, unless someone verifies it empirically. But Theorem~\ref{thm:fine-tuning} says, as long as the model is ideal, the feature layer contains all the information of $\cat$ for any downstream tasks.

\section{Multimodal Learning}
\label{subsec:multi-model}

In the previous section, we have seen how the foundation models are related to learning a category defined by a pretext task. What happens if we have multiple categories? We can first use the pretext tasks to learn different foundation models separately, then connect these models together in the embedding space. This is exactly how multimodal models like CLIP~\cite{radford2021learning}, Dall-E 2~\cite{ramesh2022hierarchical}, 
Multilingual CLIP~\cite{carlsson2022cross} and AltCLIP~\cite{chen2022altclip} work. 

In this section, we analyze the functors between different categories. Similar to the previous section, we consider the case that the functors are perfectly learned, and investigate the implications under this assumption. This setting hides unnecessary details like how the loss is defined or how network structure is designed. However, it provides interesting insights of how to connect different categories together, and what to expect after the functor connection.

As the starting point, we would like to emphasize that we assume the categories we consider are ``natural'', which means they are not only self-consistent, but also consistent with each other. For example, when we use texts like ``white, red, blue'' or ``big, small, tiny'' to describe a chair in the language category, there are corresponding images in the image category. Such connection can be described by functors. 

\subsection{Generalization Theorem}

We first assume that a functor between embedding spaces learns the object mapping between two categories perfectly, and then show that the structural information is preserved by this functor. Below we use the notations $h_\mathcal{B}, h_\cat$ and $\mathcal{B}^\vee, \cat^\vee$ to denote the ideal foundation models and the corresponding embedding spaces.

\begin{mydef}[feature-aligned functor]
Given two categories $\mathcal{B}, \cat$,  
a full embedding $F: \mathcal{C}\rightarrow \mathcal{B}$, 
denote the corresponding foundation model as $h_\mathcal{B}, h_\mathcal{C}$. 
A functor $\hat F: \mathcal{C}^\vee\rightarrow \mathcal{B}^\vee$ is  feature-aligned with $F$  if for any $X\in \ob(\mathcal{C})$, $\hat F(h_\mathcal{C}(X))\simeq h_\mathcal{B} (F(X))$. 
\end{mydef}

\begin{mythm}[Generalization theorem for structural learning]
	\label{thm:generalization}
	Consider two categories $\mathcal{B}$ and $\cat$, and a full embedding
	\[
	F:\cat\rightarrow \mathcal{B}.
	\]
	In the learning scenario, an ideal foundation model $h_\cat$ for $\cat$ together with a feature-aligned functor
	\[
	\hat F:\cat^\vee\rightarrow \mathcal{B}^\vee
	\]
	preserves the structure of $\cat$ inside $\mathcal{B}$: for any $X,Y\in \cat$,
	\[
	\hom_\cat(X,Y)\simeq
	\hom_{\mathcal{B}^\vee}\!\bigl(\hat F(h_\cat(X)),\,\hat F(h_\cat(Y))\bigr).
	\]
	Moreover, for any $X\in\cat$, the object $\hat F(h_\cat(X))\in \mathcal{B}^\vee$ lies in the representable image of $h_\mathcal{B}$. Hence there exists an object $Y_X\in \mathcal{B}$, unique up to  isomorphism, such that
	\[
	h_\mathcal{B}(Y_X)\simeq \hat F(h_\cat(X)).
	\]
	In fact, one may take $Y_X\simeq F(X)$.
\end{mythm}

\begin{proof}
	Since $F:\cat\to\mathcal{B}$ is a full embedding, for any $X,Y\in\cat$ we have
	\[
	\hom_\cat(X,Y)\simeq \hom_\mathcal{B}(F(X),F(Y)).
	\]
	By the full faithfulness of the covariant Yoneda embedding, we further obtain
	\[
	\hom_\mathcal{B}(F(X),F(Y))
	\simeq
	\hom_{\mathcal{B}^\vee}(h_\mathcal{B}(F(X)),h_\mathcal{B}(F(Y))).
	\]
	Since $\hat F$ is feature-aligned, we have
	\[
	h_\mathcal{B}(F(X))\simeq \hat F(h_\cat(X)),
	\qquad
	h_\mathcal{B}(F(Y))\simeq \hat F(h_\cat(Y)).
	\]
	Hence
	\[
	\hom_\cat(X,Y)\simeq
	\hom_{\mathcal{B}^\vee}\!\bigl(\hat F(h_\cat(X)),\,\hat F(h_\cat(Y))\bigr).
	\]
	
	The representability statement follows immediately from feature alignment:
	\[
	\hat F(h_\cat(X))\simeq h_\mathcal{B}(F(X)).
	\]
	Therefore $\hat F(h_\cat(X))$ lies in the representable image of $h_\mathcal{B}$. By the uniqueness of representatives up to isomorphism, there exists an object $Y_X\in\mathcal{B}$, unique up to isomorphism, such that
	\[
	h_\mathcal{B}(Y_X)\simeq \hat F(h_\cat(X)).
	\]
	In particular, one may take $Y_X\simeq F(X)$.
\end{proof}

Theorem~\ref{thm:generalization} is much more powerful than it appears. We call it the generalization theorem, because it provides another kind of generalization, different from the existing generalization theory on stability or Rademacher complexity. It tells us that, the structural information of one category can be recovered in the feature space by the structural information of another category with a feature-aligned functor.

\subsection{Application of Theorem~\ref{thm:generalization}}
In this subsection, we apply Theorem~\ref{thm:generalization} to analyze two models: CLIP and Dall-E 2. 

\textbf{CLIP.} The dataset of CLIP contains millions of image-text pairs. During pretraining, for each batch of $N$ pairs of data points, CLIP uses an image encoder and a text encoder to obtain $N$ pairs of embeddings, and learns a function for matching the $N$ correct pairs out of the $N\times N$ possible connections. Theorem~\ref{thm:generalization} then gives the following corollary.

\begin{mycor}[Creativity of CLIP]
	\label{cor:creativity}
	Let $\cat$ be a category of language descriptions that admit corresponding images, viewed as a full subcategory of a language category $\cat'$. Let $\mathcal{B}$ be the image category. Assume CLIP learns a feature-aligned functor
	\[
	\hat F:\cat^\vee\rightarrow \mathcal{B}^\vee
	\]
	for a full embedding
	\[
	F:\cat\rightarrow \mathcal{B}.
	\]
	Then for any $X\in\cat$, the object
	\[
	\hat F(h_\cat(X))\in \mathcal{B}^\vee
	\]
	lies in the representable image of $h_\mathcal{B}$. Hence there exists an image object $Y_X\in\mathcal{B}$, unique up to isomorphism, such that
	\[
	h_\mathcal{B}(Y_X)\simeq \hat F(h_\cat(X)).
	\]
	In particular, one may take $Y_X\simeq F(X)$. Therefore, CLIP can represent new images that can be described in $\cat$, even if they do not appear in the training set.
\end{mycor}

\textbf{Remark.} Corollary~\ref{cor:creativity} explains why CLIP-style models can support images such as ``avocado chair'', even when such images do not explicitly appear in the dataset. The full embedding assumption is crucial: otherwise, the feature computed in $\cat^\vee$ cannot be transferred faithfully to the image side $\mathcal{B}^\vee$. Strictly speaking, Corollary~\ref{cor:creativity} shows that the corresponding image object exists in the representable image of $h_\mathcal{B}$. To generate a concrete image, one still needs an image-side decoder or generator.

\textbf{Dall-E 2. } DALL-E~2 is the combination of CLIP and a diffusion model~\citep{rombach2022high, sohl2015deep, dhariwal2021diffusion, ho2020denoising}. If CLIP can be seen as a feature-aligned functor between suitable subcategories, why do we still need an extra diffusion model? The reason is that these two categories are not purely isomorphic at the object level. When we type ``a photo of dog'', there are millions of different matching images. Therefore, DALL-E~2 effectively modifies the image side so that each object is no longer a single image, but a probability distribution of images. From this perspective, the diffusion model is a generator that samples concrete images from such an image-distribution object, while CLIP learns the functor from the category of texts to the category of image distributions.

\subsection{Compositional Theorem}
By applying Theorem~\ref{thm:generalization} multiple times through a list of categories, we immediately get the following theorem. 

\begin{mythm}[Compositional Theorem]
	\label{thm:compose}
	Consider a list of categories $\{\mathcal{B}_i\}_{i=1}^n$, and $n-1$ full embeddings $\{F_i\}_{i=1}^{n-1}$, where
	\[
	F_i:\mathcal{A}_i\rightarrow \mathcal{B}_{i+1},
	\]
	with $\mathcal{A}_1=\mathcal{B}_1$, and for $i>1$, $\mathcal{A}_i$ is the full subcategory of $\mathcal{B}_i$ induced by $F_{i-1}$. Denote the foundation model for $\mathcal{B}_i$ by $h_{\mathcal{B}_i}$, and the feature-aligned functors by $\{\hat F_i\}_{i=1}^{n-1}$, where
	\[
	\hat F_i:\mathcal{B}_i^\vee\rightarrow \mathcal{B}_{i+1}^\vee.
	\]
	Define
	\[
	F\triangleq F_{n-1}F_{n-2}\cdots F_1,
	\qquad
	\hat F\triangleq \hat F_{n-1}\hat F_{n-2}\cdots \hat F_1.
	\]
	Then, for any $X,Y\in \mathcal{B}_1$,
	\[
	\hom_{\mathcal{B}_1}(X,Y)
	\simeq
	\hom_{\mathcal{B}_n^\vee}\!\bigl(\hat F(h_{\mathcal{B}_1}(X)),\,\hat F(h_{\mathcal{B}_1}(Y))\bigr).
	\]
	Moreover, for any $X\in\mathcal{B}_1$, the object
	\[
	\hat F(h_{\mathcal{B}_1}(X))\in \mathcal{B}_n^\vee
	\]
	lies in the representable image of $h_{\mathcal{B}_n}$. Hence there exists an object $Y_X\in\mathcal{B}_n$, unique up to isomorphism, such that
	\[
	h_{\mathcal{B}_n}(Y_X)\simeq \hat F(h_{\mathcal{B}_1}(X)).
	\]
	In particular, one may take $Y_X\simeq F(X)$.
\end{mythm}

\begin{proof}
	Apply Theorem~\ref{thm:generalization} recursively along the chain
	\[
	\mathcal{B}_1 \xrightarrow{F_1} \mathcal{B}_2 \xrightarrow{F_2} \cdots \xrightarrow{F_{n-1}} \mathcal{B}_n.
	\]
	At each step, the corresponding feature-aligned functor preserves the morphism structure in the next embedding space. Composing these identifications yields
	\[
	\hom_{\mathcal{B}_1}(X,Y)
	\simeq
	\hom_{\mathcal{B}_n^\vee}\!\bigl(\hat F(h_{\mathcal{B}_1}(X)),\,\hat F(h_{\mathcal{B}_1}(Y))\bigr).
	\]
	The representability statement also follows recursively: at the final step, the composed feature
	\[
	\hat F(h_{\mathcal{B}_1}(X))
	\]
	is identified with
	\[
	h_{\mathcal{B}_n}(F(X)),
	\]
	and therefore lies in the representable image of $h_{\mathcal{B}_n}$. By uniqueness of representatives up to isomorphism, there exists an object $Y_X\in\mathcal{B}_n$, unique up to isomorphism, such that
	\[
	h_{\mathcal{B}_n}(Y_X)\simeq \hat F(h_{\mathcal{B}_1}(X)).
	\]
	In particular, one may take $Y_X\simeq F(X)$.
\end{proof}

If we want to map objects from $\mathcal{B}_1$ to $\mathcal{B}_n$, but the direct training data between the two categories is limited, Theorem~\ref{thm:compose} provides an alternative route. It suffices to find a path between the two categories and learn the functors for each edge of the path. This is especially useful when some of the categories are pretrained, such as GPT or CLIP. For example, \citet{carlsson2022cross} and \citet{chen2022altclip} use multilingual text pairs to train the functors between various language categories and the English category, which are then connected to the image category using CLIP. In this way, they naturally extend CLIP to multilingual settings.

\section{Discussion}
\label{sec:discussion}
\subsection{Applying to small categories}
In our abstract and introduction, we raised the question about an infinitely large model with infinite resources, which makes our theorem look unrealistic. However, these assumptions are purely rhetoric, and our real assumption is simply the model being ideal, i.e., the model perfectly solves the pretext task. This assumption is much weaker than the infinite one, because it means our theorems can be directly applied to small datasets, which correspond to smaller categories. For instance, consider a dataset containing only 100k sentences, and a model is trained to learn the data using a language model where each sentence is connected to its neighboring sentences with certain probabilities. While the sentences or words in the corresponding category C may not be as informative as the language category that humans possess, it is well-defined, and the ideal model is also well-defined. Training an ideal model for a small dataset of this scale is feasible. According to Theorem~\ref{thm:prompt-tuning}, such a model can only perfectly solve tasks that are representable by objects in $\cat$.

\subsection{Bayesian optimal classifier}
\label{subsec:bayesian}
The notion of Bayesian optimal classifier is widely used in the generalization theory for supervised learning. It simply means that due to the inherent noise in the labels, even the best model will not be able to get zero loss in the population distribution. For example, given a blurry image of a dog, it might be difficult to tell its exact breed. 

However, in our categorical framework, we take the physicist's view of the world, and assume that \textbf{all the objects and relationships are self consistent}, without any noise. For example, 
Alaskan Husky and Siberian Husky indeed look similar from their photos, but they are fundamentally different breeds, in terms of their origins, size, appearance, temperament and so on. Our framework models these conceptual differences, and treats the photos of the dogs (which have noise) as the outcome of the last step of diffusion as discussed in Section~\ref{subsec:multi-model}. If the reader is familiar with the functional programming, the categorical framework  can be seen as the ``pure functional'' part, which is predictable and precise. By contrast, the process that deals with the actual input and output, can be seen as the ``side effects'', where the notion of Bayesian optimal classifier comes in.

\subsection{Relationship to RKHS}
Our categorical framework can be seen as a natural generalization of the RKHS framework, where in RKHS a kernel function outputs a value in $\mathbb{R}$, while our framework generalizes this to $\set$. 

For example, by analogy with the Yoneda perspective to RKHS, we obtain the reproducing property
\[
\forall x\in \cat,\ \forall f\in \mathcal{H},\ \langle k(x,\cdot), f\rangle_{\mathcal{H}} = f(x),
\]
where $f$ can be viewed as a functor in $\cat^\vee$, $k(x,\cdot)$ is a representable functor, and $\langle \cdot, \cdot \rangle_\mathcal{H}$ computes $\hom_\cat$. Therefore, in RKHS the morphisms between two objects are represented by a single real number computed by $\langle \cdot, \cdot \rangle_\mathcal{H}$. This case works perfectly for algorithms like SimCLR, where the relationship between two objects is a real number representing similarity. However, in practice, the relationship between two objects can be much more than a real number, especially for NLP tasks. In that case, RKHS is insufficient, and the general Yoneda lemma is necessary.

\section*{Acknowledgements}
This paper was greatly inspired by the fruitful interactions with the Qianfang (functor) team at the Beijing Academy of Artificial Intelligence. 
We would like to extend special thanks to Yue Cao for insightful discussions and the anonymous reviewers for their constructive feedback and suggestions. Our sincere appreciation also goes to Dun Liang and Yiyang Jia for their help with clarifications of the underlying category of SimCLR, and to Haozhe Jiang for 
suggesting a better narrative of the paper, as well as 
identifying a bug in the application of Theorem~\ref{thm:prompt-tuning} to large language models.
We also acknowledge the assistance of large language models in improving the exposition and helping identify several issues in earlier drafts of the manuscript. 
This paper is supported by the Ministry of Science and Technology of the People's Republic of China, the 2030 Innovation Megaprojects ``Program on New Generation Artificial Intelligence'' (Grant No.~2021AAA0150000).

\bibliographystyle{apalike}
\bibliography{paper}  %%% Uncomment this line and comment out the ``thebibliography'' section below to use the external .bib file (using bibtex) .

@book{mac2013categories,
	title={Categories for the working mathematician},
	author={Mac Lane, Saunders},
	volume={5},
	year={2013},
	publisher={Springer Science \& Business Media}
}

@book{riehl2017category,
	title={Category theory in context},
	author={Riehl, Emily},
	year={2017},
	publisher={Courier Dover Publications}
}

@book{adamek1990abstract,
	title={Abstract and concrete categories},
	author={Ad{\'a}mek, Ji{\v{r}}{\'\i} and Herrlich, Horst and Strecker, George},
	year={1990},
	publisher={Wiley-Interscience}
}

@book{kashiwara2006categories,
	title={Categories and Sheaves},
	author={Masaki Kashiwara, Pierre Schapira},
	year={2006},
	publisher={Springer}
}

@inproceedings{chen2020simple,
	title={A simple framework for contrastive learning of visual representations},
	author={Chen, Ting and Kornblith, Simon and Norouzi, Mohammad and Hinton, Geoffrey},
	booktitle={International conference on machine learning},
	pages={1597--1607},
	year={2020},
	organization={PMLR}
}

@article{haochen2021provable,
	title={Provable guarantees for self-supervised deep learning with spectral contrastive loss},
	author={HaoChen, Jeff Z and Wei, Colin and Gaidon, Adrien and Ma, Tengyu},
	journal={Advances in Neural Information Processing Systems},
	volume={34},
	pages={5000--5011},
	year={2021}
}

@inproceedings{he2020momentum,
	title={Momentum contrast for unsupervised visual representation learning},
	author={He, Kaiming and Fan, Haoqi and Wu, Yuxin and Xie, Saining and Girshick, Ross},
	booktitle={Proceedings of the IEEE/CVF conference on computer vision and pattern recognition},
	pages={9729--9738},
	year={2020}
}

@inproceedings{he2022masked,
	title={Masked autoencoders are scalable vision learners},
	author={He, Kaiming and Chen, Xinlei and Xie, Saining and Li, Yanghao and Doll{\'a}r, Piotr and Girshick, Ross},
	booktitle={Proceedings of the IEEE/CVF Conference on Computer Vision and Pattern Recognition},
	pages={16000--16009},
	year={2022}
}

@article{devlin2018bert,
	title={Bert: Pre-training of deep bidirectional transformers for language understanding},
	author={Devlin, Jacob and Chang, Ming-Wei and Lee, Kenton and Toutanova, Kristina},
	journal={arXiv preprint arXiv:1810.04805},
	year={2018}
}

@article{brown2020language,
	title={Language models are few-shot learners},
	author={Brown, Tom and Mann, Benjamin and Ryder, Nick and Subbiah, Melanie and Kaplan, Jared D and Dhariwal, Prafulla and Neelakantan, Arvind and Shyam, Pranav and Sastry, Girish and Askell, Amanda and others},
	journal={Advances in neural information processing systems},
	volume={33},
	pages={1877--1901},
	year={2020}
}

@article{radford2018improving,
	title={Improving language understanding by generative pre-training},
	author={Radford, Alec and Narasimhan, Karthik and Salimans, Tim and Sutskever, Ilya and others},
	year={2018},
	publisher={OpenAI},
    journal={OpenAI blog},
}

@article{radford2019language,
	title={Language models are unsupervised multitask learners},
	author={Radford, Alec and Wu, Jeffrey and Child, Rewon and Luan, David and Amodei, Dario and Sutskever, Ilya and others},
	journal={OpenAI blog},
	volume={1},
	number={8},
	pages={9},
	year={2019}
}

@inproceedings{pathak2016context,
	title={Context encoders: Feature learning by inpainting},
	author={Pathak, Deepak and Krahenbuhl, Philipp and Donahue, Jeff and Darrell, Trevor and Efros, Alexei A},
	booktitle={Proceedings of the IEEE conference on computer vision and pattern recognition},
	pages={2536--2544},
	year={2016}
}

@article{vaswani2017attention,
	title={Attention is all you need},
	author={Vaswani, Ashish and Shazeer, Noam and Parmar, Niki and Uszkoreit, Jakob and Jones, Llion and Gomez, Aidan N and Kaiser, {\L}ukasz and Polosukhin, Illia},
	journal={Advances in neural information processing systems},
	volume={30},
	year={2017}
}

@article{dosovitskiy2020image,
	title={An image is worth 16x16 words: Transformers for image recognition at scale},
	author={Dosovitskiy, Alexey and Beyer, Lucas and Kolesnikov, Alexander and Weissenborn, Dirk and Zhai, Xiaohua and Unterthiner, Thomas and Dehghani, Mostafa and Minderer, Matthias and Heigold, Georg and Gelly, Sylvain and others},
	journal={arXiv preprint arXiv:2010.11929},
	year={2020}
}

@inproceedings{ramesh2021zero,
	title={Zero-shot text-to-image generation},
	author={Ramesh, Aditya and Pavlov, Mikhail and Goh, Gabriel and Gray, Scott and Voss, Chelsea and Radford, Alec and Chen, Mark and Sutskever, Ilya},
	booktitle={International Conference on Machine Learning},
	pages={8821--8831},
	year={2021},
	organization={PMLR}
}

@inproceedings{chen2021exploring,
	title={Exploring simple siamese representation learning},
	author={Chen, Xinlei and He, Kaiming},
	booktitle={Proceedings of the IEEE/CVF Conference on Computer Vision and Pattern Recognition},
	pages={15750--15758},
	year={2021}
}

@article{clark2020electra,
	title={Electra: Pre-training text encoders as discriminators rather than generators},
	author={Clark, Kevin and Luong, Minh-Thang and Le, Quoc V and Manning, Christopher D},
	journal={arXiv preprint arXiv:2003.10555},
	year={2020}
}

@inproceedings{doersch2015unsupervised,
	title={Unsupervised visual representation learning by context prediction},
	author={Doersch, Carl and Gupta, Abhinav and Efros, Alexei A},
	booktitle={Proceedings of the IEEE international conference on computer vision},
	pages={1422--1430},
	year={2015}
}

@article{gidaris2018unsupervised,
	title={Unsupervised representation learning by predicting image rotations},
	author={Gidaris, Spyros and Singh, Praveer and Komodakis, Nikos},
	journal={arXiv preprint arXiv:1803.07728},
	year={2018}
}

@article{grill2020bootstrap,
	title={Bootstrap your own latent-a new approach to self-supervised learning},
	author={Grill, Jean-Bastien and Strub, Florian and Altch{\'e}, Florent and Tallec, Corentin and Richemond, Pierre and Buchatskaya, Elena and Doersch, Carl and Avila Pires, Bernardo and Guo, Zhaohan and Gheshlaghi Azar, Mohammad and others},
	journal={Advances in neural information processing systems},
	volume={33},
	pages={21271--21284},
	year={2020}
}

@inproceedings{noroozi2016unsupervised,
	title={Unsupervised learning of visual representations by solving jigsaw puzzles},
	author={Noroozi, Mehdi and Favaro, Paolo},
	booktitle={European conference on computer vision},
	pages={69--84},
	year={2016},
	organization={Springer}
}

@article{oord2018representation,
	title={Representation learning with contrastive predictive coding},
	author={Oord, Aaron van den and Li, Yazhe and Vinyals, Oriol},
	journal={arXiv preprint arXiv:1807.03748},
	year={2018}
}

@inproceedings{pathak2017learning,
	title={Learning features by watching objects move},
	author={Pathak, Deepak and Girshick, Ross and Doll{\'a}r, Piotr and Darrell, Trevor and Hariharan, Bharath},
	booktitle={Proceedings of the IEEE conference on computer vision and pattern recognition},
	pages={2701--2710},
	year={2017}
}

@article{raffel2020exploring,
	title={Exploring the limits of transfer learning with a unified text-to-text transformer.},
	author={Raffel, Colin and Shazeer, Noam and Roberts, Adam and Lee, Katherine and Narang, Sharan and Matena, Michael and Zhou, Yanqi and Li, Wei and Liu, Peter J and others},
	journal={J. Mach. Learn. Res.},
	volume={21},
	number={140},
	pages={1--67},
	year={2020}
}

@inproceedings{barlow,
	author    = {Jure Zbontar and
	Li Jing and
	Ishan Misra and
	Yann LeCun and
	St{\'{e}}phane Deny},
	editor    = {Marina Meila and
	Tong Zhang},
	title     = {Barlow Twins: Self-Supervised Learning via Redundancy Reduction},
	booktitle = {Proceedings of the 38th International Conference on Machine Learning,
	{ICML} 2021, 18-24 July 2021, Virtual Event},
	series    = {Proceedings of Machine Learning Research},
	volume    = {139},
	pages     = {12310--12320},
	publisher = {{PMLR}},
	year      = {2021},
}

@inproceedings{arandjelovic2018objects,
	title={Objects that sound},
	author={Arandjelovic, Relja and Zisserman, Andrew},
	booktitle={Proceedings of the European conference on computer vision (ECCV)},
	pages={435--451},
	year={2018}
}

@inproceedings{rombach2022high,
	title={High-resolution image synthesis with latent diffusion models},
	author={Rombach, Robin and Blattmann, Andreas and Lorenz, Dominik and Esser, Patrick and Ommer, Bj{\"o}rn},
	booktitle={Proceedings of the IEEE/CVF Conference on Computer Vision and Pattern Recognition},
	pages={10684--10695},
	year={2022}
}

@article{ramesh2022hierarchical,
	title={Hierarchical text-conditional image generation with clip latents},
	author={Ramesh, Aditya and Dhariwal, Prafulla and Nichol, Alex and Chu, Casey and Chen, Mark},
	journal={arXiv preprint arXiv:2204.06125},
	year={2022}
}

@inproceedings{sohl2015deep,
	title={Deep unsupervised learning using nonequilibrium thermodynamics},
	author={Sohl-Dickstein, Jascha and Weiss, Eric and Maheswaranathan, Niru and Ganguli, Surya},
	booktitle={International Conference on Machine Learning},
	pages={2256--2265},
	year={2015},
	organization={PMLR}
}

@article{dhariwal2021diffusion,
	title={Diffusion models beat gans on image synthesis},
	author={Dhariwal, Prafulla and Nichol, Alexander},
	journal={Advances in Neural Information Processing Systems},
	volume={34},
	pages={8780--8794},
	year={2021}
}

@article{ho2020denoising,
	title={Denoising diffusion probabilistic models},
	author={Ho, Jonathan and Jain, Ajay and Abbeel, Pieter},
	journal={Advances in Neural Information Processing Systems},
	volume={33},
	pages={6840--6851},
	year={2020}
}

@article{jiang2020can,
	title={How can we know what language models know?},
	author={Jiang, Zhengbao and Xu, Frank F and Araki, Jun and Neubig, Graham},
	journal={Transactions of the Association for Computational Linguistics},
	volume={8},
	pages={423--438},
	year={2020},
	publisher={MIT Press}
}

@article{shin2020autoprompt,
	title={Autoprompt: Eliciting knowledge from language models with automatically generated prompts},
	author={Shin, Taylor and Razeghi, Yasaman and Logan IV, Robert L and Wallace, Eric and Singh, Sameer},
	journal={arXiv preprint arXiv:2010.15980},
	year={2020}
}

@article{gao2020making,
	title={Making pre-trained language models better few-shot learners},
	author={Gao, Tianyu and Fisch, Adam and Chen, Danqi},
	journal={arXiv preprint arXiv:2012.15723},
	year={2020}
}

@article{liu2021gpt,
	title={GPT understands, too},
	author={Liu, Xiao and Zheng, Yanan and Du, Zhengxiao and Ding, Ming and Qian, Yujie and Yang, Zhilin and Tang, Jie},
	journal={arXiv preprint arXiv:2103.10385},
	year={2021}
}

@article{li2021prefix,
	title={Prefix-tuning: Optimizing continuous prompts for generation},
	author={Li, Xiang Lisa and Liang, Percy},
	journal={arXiv preprint arXiv:2101.00190},
	year={2021}
}

@inproceedings{radford2021learning,
	title={Learning transferable visual models from natural language supervision},
	author={Radford, Alec and Kim, Jong Wook and Hallacy, Chris and Ramesh, Aditya and Goh, Gabriel and Agarwal, Sandhini and Sastry, Girish and Askell, Amanda and Mishkin, Pamela and Clark, Jack and others},
	booktitle={International Conference on Machine Learning},
	pages={8748--8763},
	year={2021},
	organization={PMLR}
}

@inproceedings{du2019gradient,
	title={Gradient descent finds global minima of deep neural networks},
	author={Du, Simon and Lee, Jason and Li, Haochuan and Wang, Liwei and Zhai, Xiyu},
	booktitle={International conference on machine learning},
	pages={1675--1685},
	year={2019},
	organization={PMLR}
}

@inproceedings{arora2019fine,
	title={Fine-grained analysis of optimization and generalization for overparameterized two-layer neural networks},
	author={Arora, Sanjeev and Du, Simon and Hu, Wei and Li, Zhiyuan and Wang, Ruosong},
	booktitle={International Conference on Machine Learning},
	pages={322--332},
	year={2019},
	organization={PMLR}
}

@article{du2018gradient,
	title={Gradient descent provably optimizes over-parameterized neural networks},
	author={Du, Simon S and Zhai, Xiyu and Poczos, Barnabas and Singh, Aarti},
	journal={arXiv preprint arXiv:1810.02054},
	year={2018}
}

@inproceedings{du2019provably,
	title={Provably efficient RL with rich observations via latent state decoding},
	author={Du, Simon and Krishnamurthy, Akshay and Jiang, Nan and Agarwal, Alekh and Dudik, Miroslav and Langford, John},
	booktitle={International Conference on Machine Learning},
	pages={1665--1674},
	year={2019},
	organization={PMLR}
}

@article{du2019good,
	title={Is a good representation sufficient for sample efficient reinforcement learning?},
	author={Du, Simon S and Kakade, Sham M and Wang, Ruosong and Yang, Lin F},
	journal={arXiv preprint arXiv:1910.03016},
	year={2019}
}

@inproceedings{allen2019convergence,
	title={A convergence theory for deep learning via over-parameterization},
	author={Allen-Zhu, Zeyuan and Li, Yuanzhi and Song, Zhao},
	booktitle={International Conference on Machine Learning},
	pages={242--252},
	year={2019},
	organization={PMLR}
}

@article{allen2019learning,
	title={Learning and generalization in overparameterized neural networks, going beyond two layers},
	author={Allen-Zhu, Zeyuan and Li, Yuanzhi and Liang, Yingyu},
	journal={Advances in neural information processing systems},
	volume={32},
	year={2019}
}

@article{jin2018q,
	title={Is Q-learning provably efficient?},
	author={Jin, Chi and Allen-Zhu, Zeyuan and Bubeck, Sebastien and Jordan, Michael I},
	journal={Advances in neural information processing systems},
	volume={31},
	year={2018}
}

@inproceedings{cai2020provably,
	title={Provably efficient exploration in policy optimization},
	author={Cai, Qi and Yang, Zhuoran and Jin, Chi and Wang, Zhaoran},
	booktitle={International Conference on Machine Learning},
	pages={1283--1294},
	year={2020},
	organization={PMLR}
}

@article{zou2020gradient,
	title={Gradient descent optimizes over-parameterized deep ReLU networks},
	author={Zou, Difan and Cao, Yuan and Zhou, Dongruo and Gu, Quanquan},
	journal={Machine learning},
	volume={109},
	number={3},
	pages={467--492},
	year={2020},
	publisher={Springer}
}

@article{bartlett2017spectrally,
	title={Spectrally-normalized margin bounds for neural networks},
	author={Bartlett, Peter L and Foster, Dylan J and Telgarsky, Matus J},
	journal={Advances in neural information processing systems},
	volume={30},
	year={2017}
}

@article{bartlett2020benign,
	title={Benign overfitting in linear regression},
	author={Bartlett, Peter L and Long, Philip M and Lugosi, G{\'a}bor and Tsigler, Alexander},
	journal={Proceedings of the National Academy of Sciences},
	volume={117},
	number={48},
	pages={30063--30070},
	year={2020},
	publisher={National Acad Sciences}
}

@inproceedings{yin2019rademacher,
	title={Rademacher complexity for adversarially robust generalization},
	author={Yin, Dong and Kannan, Ramchandran and Bartlett, Peter},
	booktitle={International conference on machine learning},
	pages={7085--7094},
	year={2019},
	organization={PMLR}
}

@article{li2017convergence,
	title={Convergence analysis of two-layer neural networks with relu activation},
	author={Li, Yuanzhi and Yuan, Yang},
	journal={Advances in neural information processing systems},
	volume={30},
	year={2017}
}

@article{hornik1989multilayer,
	title={Multilayer feedforward networks are universal approximators},
	author={Hornik, Kurt and Stinchcombe, Maxwell and White, Halbert},
	journal={Neural networks},
	volume={2},
	number={5},
	pages={359--366},
	year={1989},
	publisher={Elsevier}
}

@article{cybenko1989approximation,
	title={Approximation by superpositions of a sigmoidal function},
	author={Cybenko, George},
	journal={Mathematics of control, signals and systems},
	volume={2},
	number={4},
	pages={303--314},
	year={1989},
	publisher={Springer}
}

@inproceedings{raghu2017expressive,
	title={On the expressive power of deep neural networks},
	author={Raghu, Maithra and Poole, Ben and Kleinberg, Jon and Ganguli, Surya and Sohl-Dickstein, Jascha},
	booktitle={international conference on machine learning},
	pages={2847--2854},
	year={2017},
	organization={PMLR}
}

@article{censi2015mathematical,
	title={A mathematical theory of co-design},
	author={Censi, Andrea},
	journal={arXiv preprint arXiv:1512.08055},
	year={2015}
}

@book{marquis2008geometrical,
	title={From a geometrical point of view: A study of the history and philosophy of category theory},
	author={Marquis, Jean-Pierre},
	volume={14},
	year={2008},
	publisher={Springer Science \& Business Media}
}

@book{kus2019category,
	title={Category Theory in Physics, Mathematics, and Philosophy},
	author={Ku{\'s}, Marek and Skowron, Bart{\l}omiej},
	year={2019},
	publisher={Springer}
}

@article{shiebler2021category,
	title={Category theory in machine learning},
	author={Shiebler, Dan and Gavranovi{\'c}, Bruno and Wilson, Paul},
	journal={arXiv preprint arXiv:2106.07032},
	year={2021}
}

@article{fong2018seven,
	title={Seven sketches in compositionality: An invitation to applied category theory},
	author={Fong, Brendan and Spivak, David I},
	journal={arXiv preprint arXiv:1803.05316},
	year={2018}
}

@article{bradley2022enriched,
	title={An enriched category theory of language: from syntax to semantics},
	author={Bradley, Tai-Danae and Terilla, John and Vlassopoulos, Yiannis},
	journal={La Matematica},
	pages={1--30},
	year={2022},
	publisher={Springer}
}

@article{mahadevan2022categoroids,
	title={Categoroids: Universal Conditional Independence},
	author={Mahadevan, Sridhar},
	journal={arXiv preprint arXiv:2208.11077},
	year={2022}
}

@article{mahadevan2022unifying,
	title={Unifying Causal Inference and Reinforcement Learning using Higher-Order Category Theory},
	author={Mahadevan, Sridhar},
	journal={arXiv preprint arXiv:2209.06262},
	year={2022}
}

@inproceedings{wen2021toward,
	title={Toward understanding the feature learning process of self-supervised contrastive learning},
	author={Wen, Zixin and Li, Yuanzhi},
	booktitle={International Conference on Machine Learning},
	pages={11112--11122},
	year={2021},
	organization={PMLR}
}

@article{wen2022mechanism,
	title={The Mechanism of Prediction Head in Non-contrastive Self-supervised Learning},
	author={Wen, Zixin and Li, Yuanzhi},
	journal={arXiv preprint arXiv:2205.06226},
	year={2022}
}

@article{luo2022one,
	title={One Objective for All Models--Self-supervised Learning for Topic Models},
	author={Luo, Zeping and Weng, Cindy and Wu, Shiyou and Zhou, Mo and Ge, Rong},
	journal={arXiv preprint arXiv:2203.03539},
	year={2022}
}

@article{liu2022rectified,
	title={Rectified Flow: A Marginal Preserving Approach to Optimal Transport},
	author={Liu, Qiang},
	journal={arXiv preprint arXiv:2209.14577},
	year={2022}
}

@article{arora2019theoretical,
	title={A theoretical analysis of contrastive unsupervised representation learning},
	author={Arora, Sanjeev and Khandeparkar, Hrishikesh and Khodak, Mikhail and Plevrakis, Orestis and Saunshi, Nikunj},
	journal={arXiv preprint arXiv:1902.09229},
	year={2019}
}

@inproceedings{tosh2021contrastive,
	title={Contrastive learning, multi-view redundancy, and linear models},
	author={Tosh, Christopher and Krishnamurthy, Akshay and Hsu, Daniel},
	booktitle={Algorithmic Learning Theory},
	pages={1179--1206},
	year={2021},
	organization={PMLR}
}

@article{lee2021predicting,
	title={Predicting what you already know helps: Provable self-supervised learning},
	author={Lee, Jason D and Lei, Qi and Saunshi, Nikunj and Zhuo, Jiacheng},
	journal={Advances in Neural Information Processing Systems},
	volume={34},
	pages={309--323},
	year={2021}
}

@inproceedings{zimmermann2021contrastive,
	title={Contrastive learning inverts the data generating process},
	author={Zimmermann, Roland S and Sharma, Yash and Schneider, Steffen and Bethge, Matthias and Brendel, Wieland},
	booktitle={International Conference on Machine Learning},
	pages={12979--12990},
	year={2021},
	organization={PMLR}
}

@inproceedings{he2016deep,
	title={Deep residual learning for image recognition},
	author={He, Kaiming and Zhang, Xiangyu and Ren, Shaoqing and Sun, Jian},
	booktitle={Proceedings of the IEEE conference on computer vision and pattern recognition},
	pages={770--778},
	year={2016}
}

@book{shalev2014understanding,
	title={Understanding machine learning: From theory to algorithms},
	author={Shalev-Shwartz, Shai and Ben-David, Shai},
	year={2014},
	publisher={Cambridge university press}
}

@article{chen2022altclip,
	title={AltCLIP: Altering the Language Encoder in CLIP for Extended Language Capabilities},
	author={Chen, Zhongzhi and Liu, Guang and Zhang, Bo-Wen and Ye, Fulong and Yang, Qinghong and Wu, Ledell},
	journal={arXiv preprint arXiv:2211.06679},
	year={2022}
}

@inproceedings{carlsson2022cross,
	title={Cross-lingual and multilingual clip},
	author={Carlsson, Fredrik and Eisen, Philipp and Rekathati, Faton and Sahlgren, Magnus},
	booktitle={Proceedings of the Thirteenth Language Resources and Evaluation Conference},
	pages={6848--6854},
	year={2022}
}

@misc{tan2023contrastive,
	title={Contrastive Learning Is Spectral Clustering On Similarity Graph}, 
	author={Zhiquan Tan and Yifan Zhang and Jingqin Yang and Yang Yuan},
	year={2023},
	eprint={2303.15103},
	archivePrefix={arXiv},
	primaryClass={cs.LG}
}





% \end{thebibliography}

\end{document}